%% file: main.tex
\documentclass[letterpaper]{article} 
\usepackage{aaai2027}  
\usepackage[hyphens]{url}  
\usepackage{graphicx} 
\usepackage{natbib}  
\usepackage{caption} 
\usepackage{algorithm}
\usepackage{algorithmic}
\usepackage{subcaption}

\usepackage{amsmath}
\usepackage{amsfonts}
\usepackage[notextcomp]{stix}
\usepackage{tikz}

\usepackage{amsthm}
\theoremstyle{definition}
\newtheorem{theorem}{Theorem}
\newtheorem{proposition}{Proposition}
\newtheorem{lemma}{Lemma}
\newtheorem{corollary}{Corollary}
\newtheorem{definition}{Definition}

\newcommand{\Sol}[1]{\text{Sol}_{#1}}
\newcommand{\Mid}[2]{\text{Mid}(#1, #2)}

\usepackage{newfloat}
\usepackage{listings}
\DeclareCaptionStyle{ruled}{labelfont=normalfont,labelsep=colon,strut=off} 
\floatstyle{ruled}
\newfloat{listing}{tb}{lst}{}
\floatname{listing}{Listing}

\usepackage{booktabs}

\nocopyright

\title{A Generalized Parallelogram Rule for \\ Proportional Analogies on Riemannian Manifolds}
\author{
    Pierre-Alexandre Murena\textsuperscript{\rm 1}, 
    Marcelo Hartmann\textsuperscript{\rm 2}
}
\affiliations{
    \textsuperscript{\rm 1}Hamburg University of Technology (TUHH), Research Group on Human-Centric Machine Learning, Germany\\
    \textsuperscript{\rm 2}Department of Computer Science, University of Helsinki, Finland\\
    pierre-alexandre.murena@tuhh.de, marcelo.hartmann@helsinki.fi
}

\begin{document}

\maketitle

\begin{abstract}

\input{sections/abstract}
\end{abstract}

\begin{links}
    \link{Code}{https://github.com/ppaamm/RiemannianAnalogies}
\end{links}

\section{Introduction}
\input{sections/introduction}

\section{Proportional Analogies}
\input{sections/analogies}

\section{Riemannian Geometry}
\input{sections/riemannian}

\section{Proportional Analogies based on the Generalized Parallelogram}
\input{sections/generalized-parallelogram}

\section{Proportional Analogies on $\mathbb{R}^n$}
\input{sections/real-line}

\section{Proportional Analogies on Riemannian Symmetric Spaces}
\input{sections/symmetric-spaces}

\section{Proportional Analogies on Shape Spaces}
\input{sections/shape}

\section{Proportional Analogies on Categorical Distributions}
\input{sections/categorical-distributions}

\section{Conclusion}
\label{sec:conclusion}
\input{sections/conclusion}

\section*{Acknowledgments}

The authors would like to warmly thank Jean Lieber for the supporting discussions and feedback about this work, and Frank Nielsen for his insights on the geometry of the simplex. 

\bibliography{aaai2027}

\end{document}

%% file: sections/abstract.tex
Analogies are quaternary relations of the form ``$a$ is to $b$ as $c$ is to $d$'', usually denoted $a : b :: c : d$. This notion is formalized in particular with the notion of proportional analogy, which imposes some constraints on the valid analogies. Whereas proportional analogies have been studied mostly in symbolic domains and in vector spaces, their use is limited in non-Euclidean spaces. In this paper, we introduce a proportional analogy relation in Riemannian domains, extending the parallelogram rule used for arithmetic analogies in Euclidean spaces. We illustrate the introduced analogy on various manifolds, such as the sphere, shape spaces and manifolds of probability distributions. 

%% file: sections/introduction.tex
\begin{figure}[h]
    \centering
    \includegraphics[width=\linewidth]{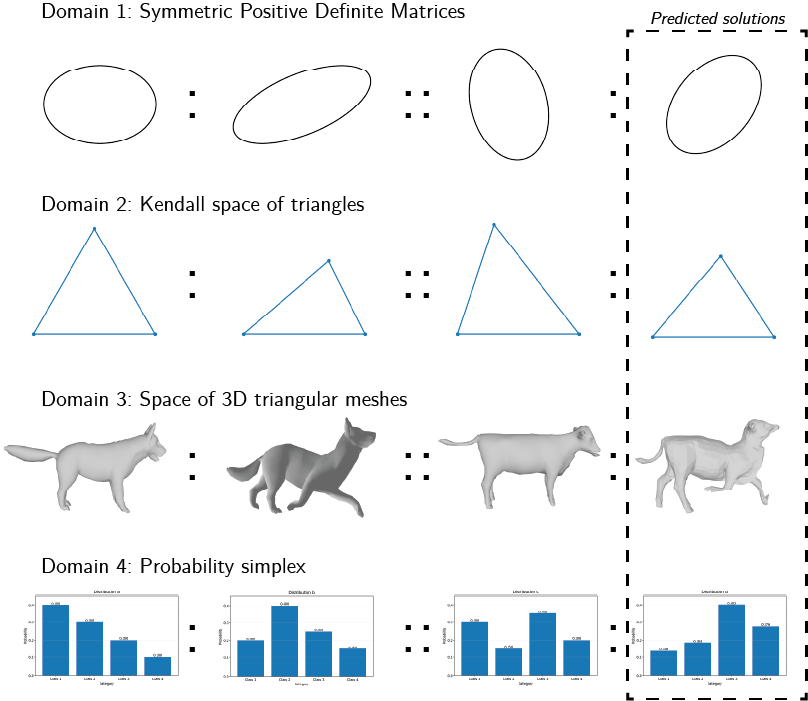}
    \caption{Analogical equations solved by our generalized parallelogram rule on Riemannian manifolds. \textbf{Domain 1} corresponds to Symmetric Positive Definite Matrices, with results presented in Corollary~\ref{corr:spd}. \textbf{Domain 2} is Kendall's space of triangular shapes, with results presented in Corollary~\ref{corr:analogies-sphere}. \textbf{Domain 3} is the space of 3D triangular meshes, using approximated geodesics and logarithmic maps. \textbf{Domain 4} is the simplex of probability distributions, endowed with the Aitchison metric, with results presented in Proposition~\ref{prop:aitchison}. The source code to generate these examples is in supplementary material, and will be published upon acceptance.}
    \label{fig:placeholder}
\end{figure}
Analogies are quaternary relations of the form ``$a$ is to $b$ as $c$ is to $d$'', usually written $a : b :: c : d$~\cite{prade2021analogical}. The capability to understand and solve analogies, i.e. to find the $d$ that makes it correct, is recognized as one of the core characteristics of human cognition~\cite{Mitchell2000AnalogyMakingAA}. 
It has also been used as a validation tool for machine learning methods. An early notable case was the use of analogies to demonstrate the validity of word embeddings such as word2vec~\cite{mikolov2013efficient}: although these models were not explicitly trained to learn analogies, it was observed that the parallelogram rule $b - a = d - c$~\cite{rumelhart1973model} applied in the learned latent space yielded valid analogies. While this finding has been widely debated~\cite{peterson2020parallelograms}, it nonetheless helped establish analogies as a tool for validation. More recently, analogies have been used to evaluate the reasoning abilities of Large Language Models (LLMs)~\cite{webb2023emergent}, and the ARC-AGI challenge considers analogy-making as a central capability of Artificial General Intelligence~\cite{chollet2019measure}.

Various models of analogies have been proposed, among which \textit{proportional analogies} have been particularly popular~\cite{lepage2004analogy}. This model is an axiomatic characterization, stating that $a : b :: c : d$ holds on some conditions of the four terms. A strength of this model is that it is consistent with the philosophical roots of analogies while applying to a large number of domains~\cite{prade2021analogical}. In particular, the parallelogram rule in vector spaces is a valid proportional analogy. However, most of the research was conducted either on symbolic domains (character strings, Boolean values, logic formulae) or in vector spaces, ignoring more complex structures such as manifolds. The preliminary work of~\citet{murena2018opening} is an attempt to extend the parallelogram rule in Riemannian manifolds, but finds two limitations: the direct generalization of the rule, based on geodesic shooting and parallel transport, does not produce a valid proportional analogy, and the general construction of proportional analogies does not guarantee continuity properties.

In this paper, we propose a different perspective. Rather than characterizing analogies through parallel displacements, we characterize them through geodesic midpoints. In Euclidean spaces, the classical parallelogram can equivalently be described by the fact that the diagonals share the same midpoint. We show that this characterization extends naturally to arbitrary Riemannian manifolds by replacing Euclidean midpoints with intrinsic geodesic midpoints. The resulting construction is entirely intrinsic, requires only exponential and logarithmic maps and satisfies the axioms of proportional analogies.
Our contributions are fourfold:
\begin{itemize}
    \item We introduce a new intrinsic definition of proportional analogies on Riemannian manifolds based on geodesic midpoints and prove that it satisfies the axioms of proportional analogies. 
    \item We establish several theoretical properties of the proposed framework, including invariance under Riemannian isometries, robustness, and explicit characterizations on geodesically convex and symmetric spaces.
    \item We extend the results of \citet{lepage2024any}, showing that any quadruple of pairwise distinct points of $\mathbb{R}^n$ can be made analogous for a suitable metric when $n \geq 2$, and proposing ordering conditions for $n = 1$.
    \item We illustrate the versatility of the framework on diverse geometric domains (Figure~\ref{fig:placeholder}), demonstrating that a single intrinsic construction naturally extends analogical reasoning far beyond Euclidean vector spaces.
\end{itemize}

%% file: sections/analogies.tex
An analogy is a quaternary relation, written ``$A : B :: C : D$" and which reads ``A is to B as C is to D". A common characterization of analogies, inspired by the philosophical works on proportions of Plato and Aristotle, relies on the satisfaction of three properties: 

\begin{definition}[Proportional analogy]
A quaternary relation of the form $a : b :: c : d$ is a proportional analogy if it satisfies the following properties for all $a, b, c, d$:
\begin{itemize}
    \item $a : b :: c : d \Leftrightarrow c : d :: a : b$ \qquad \textit{(symmetry)}
    \item $a : b :: c : d \Leftrightarrow a : c :: b : d$ \qquad \textit{(exchange of the means)}
    \item $a : b :: a : b$ \qquad \textit{(reflexivity)}
\end{itemize}
A proportional analogy is called \emph{strong} if it satisfies $a : b :: a : x \Rightarrow x = b$ for all $a$ and $b$. 
\end{definition}

From the first two postulates, 8 equivalent forms can be derived:
\begin{proposition}
\label{prop:equivalent-forms}
For a proportional analogy, the following forms are equivalent for any $a, b, c$ and $d$: 
\begin{align*}
    &a : b :: c : d \qquad a : c :: b : d \qquad b : d :: a : c \\
    &b : a :: d : c \qquad d : c :: b : a  \qquad d : b :: c : a \\
    &c : a :: d : b \qquad c : d :: a : b
\end{align*}
\end{proposition}

In a context of vector spaces, it is easy to verify that the \emph{parallelogram rule}~\cite{rumelhart1973model} corresponds to a proportional analogy. 

\begin{proposition}
\label{prop:arithmetic}
The relation $a : b ::_{E} c : d$ on $a, b, c, d$ in a vector space $E$ holding true when $b - a = d - c$ is a proportional analogy, called \emph{arithmetic proportion}. 
\end{proposition}
Various interpretations can be given of this rule, for instance that $b$ differs from $a$ the same way $d$ differs from $c$. Equivalent interpretations can be given by on the equivalent forms presented in Proposition~\ref{prop:equivalent-forms}. These interpretations all consist in comparing two by two the parallel sides of the parallelogram. 

An alternative vision presented by \citet{lepage2024any} takes a different perspective on the parallelogram: instead of characterizing the parallelogram by its parallel sides, they propose to describe it in terms of its center: four points $a, b, c$ and $d$ are in analogy if and only if the middle of the segment $[a,c]$ coincides with the middle of the segment $[b, d]$. The authors use the notion of generalized mean~\cite{holder1882grenzwerthe} to propose a family of analogical proportions.

\begin{proposition}
Let $m_p$ be a generalized mean of order $p \in \mathbb{R}$ defined as $m_p(a, b) = \lim_{r \rightarrow p} \sqrt[r]{\frac{a^r + b^r}{2}}$. Then the relation $a : b ::_p : c : d$ which holds true if and only if $m_p(a, c) = m_p(b, d)$ is a proportional analogy. 
\end{proposition}

An important question is to solve \emph{analogical equations}, i.e. to find $x$ such that $a : b :: c : x$ for a given proportional analogy. In the following, we use the notation 
\begin{equation}
    Sol_X(a, b, c) = \lbrace d \in X \, | \, a : b :: c : d \rbrace
\end{equation}
to designate the set of solutions of the proportional analogy in the space $X$. 

%% file: sections/riemannian.tex
A \textit{topological manifold} of dimension $d$ is defined as a paracompact Haussdorf space in which every point has a neighborhood $U$ that is homeomorphic to an open subset of $\mathbb{R}^d$, called \emph{chart}. A topological manifold is called \textit{differentiable} if the mapping from one chart to the other is smooth. 

A \emph{tangent vector} $\xi$ to a differentiable manifold $\mathcal{M}$ at point $p$ is the equivalence class of the differentiable curves $\gamma$ on $\mathcal{M}$ such that $\gamma(0) = p$ modulo a first-order contact condition between the curves. The set of all tangent vectors to a point $p \in \mathcal{M}$ is denoted $T_p\mathcal{M}$ and can be shown to be a $d$-dimensional vector space. The tangent bundle $T\mathcal{M}$ is defined as the disjoint union of all tangent vectors: $T\mathcal{M} = \bigsqcup_{p\in\mathcal{M}} T_p \mathcal{M}$

The tangent space $T_p\mathcal{M}$ can be equipped with an inner product $g_p : (T_p \mathcal{M})^2 \rightarrow \mathbb{R}$. When $g_p$ varies smoothly with $p$ and it is positive, we call $(\mathcal{M}, g)$ a \emph{Riemannian manifold}. 

A connection $\nabla$ on a smooth manifold $\mathcal{M}$ is a bilinear map $\mathcal{C}^\infty(T\mathcal{M}) \times \mathcal{C}^\infty(T\mathcal{M}) \rightarrow\mathcal{C}^\infty(T\mathcal{M})$, written $\nabla_X Y$, such that $\nabla_{fX} Y = f \, \nabla_X Y$ and $\nabla_X(fY) = X[f] Y + f \,  \nabla_X Y$ for all $X, Y \in T\mathcal{M}$ and all $f \in \mathcal{C}^\infty(\mathcal{M})$.

\begin{definition}
Let $(\mathcal{M}, g)$ be a Riemannian manifold and let $\nabla$ be a connection on $\mathcal{M}$. A smooth curve $\gamma: [0, 1] \rightarrow \mathcal{M}$ is said to be a geodesic if $\nabla_{\dot{\gamma}} \dot{\gamma} = 0$. 
\end{definition}

It can be verified that, for all $p \in \mathcal{M}$ and $V \in T_p\mathcal{M}$, there exists a unique geodesic such that $\gamma(0) = p$ and $\dot{\gamma}(0) = V$. 

\begin{definition}[Exponential map]
Let $(\mathcal{M}, g)$ be a Riemannian manifold and let $\nabla$ be a connection on $\mathcal{M}$. The exponential map at a point $p \in \mathcal{M}$ is the map $\exp_p: \mathcal{D}_p \rightarrow \mathcal{M}$, with $\mathcal{D}_p \subseteq T_p\mathcal{M}$ such that for all $V \in T_p\mathcal{M}, \exp_p(V) = \gamma(1)$ where $\gamma$ is the unique geodesic such that $\gamma(0) = p$ and $\dot{\gamma}(0) = V$. When $\mathcal{D}_p = T_p\mathcal{M}$ of all $p$, we say that $\mathcal{M}$ is \emph{geodesically complete}.
\end{definition}

\begin{proposition}
\label{prop:geodesic-expmap}
Under the same conditions as the previous definition, the path $t \mapsto \exp_p(tV)$ is the geodesic starting at $p$ and with initial velocity $V$. 
\end{proposition}

A Riemannian manifold $(\mathcal{M}, g)$ is said to be \emph{geodesically convex} when there exists a unique geodesic between any two points on $\mathcal{M}$. In this case, we define the logarithmic map, denoted by $\log_p(q)$, as the unique vector $V \in T_p\mathcal{M}$ such that $q = \exp_p(V)$. 

\begin{definition}[Arclength]
Let $\gamma:[a,b] \rightarrow \mathcal{M}$ be a smooth curve on the Riemannian manifold $(\mathcal{M}, g)$. The \emph{arclength} of $\gamma$ from $t_1$ to $t_2$ $(t_1 \leq t_2)$ is defined as:
\begin{equation}
    L(\gamma) = \int_{t_1}^{t_2} \| \dot{\gamma}(t) \| \, dt
\end{equation}
\end{definition}

\begin{proposition}
\label{prop:length-proportion-geodesic}
Let $(\mathcal{M}, g)$ be a Riemannian manifold and let $\nabla$ be a connection on $\mathcal{M}$. Given $p\in \mathcal{M}$ and $V \in T_p\mathcal{M}$, let $\gamma_\tau: [0,1] \rightarrow \mathcal{M}$ be the geodesic defined by $\gamma_\tau(t) = \exp_p(t \tau V)$ for $t \in [0, 1]$. Then $L(\gamma_\tau) = \tau L(\gamma_1)$. 
\end{proposition}
\begin{proof}
Since $\gamma_\tau$ is a geodesic, the quantity $\|\dot{\gamma}_\tau(t)\|_g$ is constant, and equal to $\| \tau V \|_g$. Then $L(\gamma_\tau) = \sqrt{g_p(\tau V, \tau V)} = \tau \sqrt{g_p(V, V)} = \tau L(\gamma_1)$. 
\end{proof}

%% file: sections/generalized-parallelogram.tex
\subsection{Proportional Analogies on Geodesically Convex Riemannian Manifolds}

The construction process of the parallelogram used in arithmetic analogies on vector spaces can be naturally extended to Riemannian manifolds~\cite{murena2018opening}. Instead of considering the vectors $b - a$ and $d - c$, we consider the \emph{logarithmic map} $\log_a(b)$, defined as a the vector $V \in T_a\mathcal{M}$ such that $\exp_a(V) = b$. Given $a, b$ and $c$, a natural solution to compute $d$ is to compute $\log_a(b)$, to transport it to $c$ and to shoot the corresponding geodesic from $c$. However, such a process does not result in a proportional analogy, since the $d$ obtained by solving $a : b :: c : x$ and $a : c :: b : x$ is not the same. This is a consequence of the curvature of the manifold.

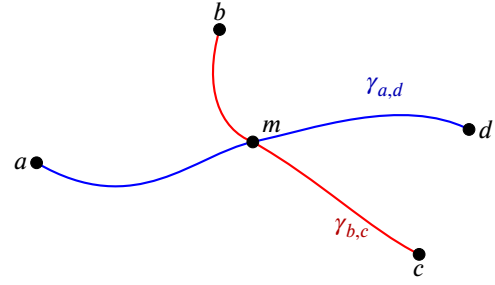
\begin{figure}
\centering
\begin{tikzpicture}[scale=1.1]

\coordinate (a) at (0,0);
\coordinate (b) at (2.2,1.6);
\coordinate (m) at (2.6,0.25);
\coordinate (d) at (5.2,0.4);
\coordinate (c) at (4.6,-1.1);

\draw[thick, blue]
  (a) .. controls (1.2,-0.7) and (1.9,0.1) .. (m)
      .. controls (3.3,0.4) and (4.4,0.8) .. (d);

\draw[thick, red]
  (b) .. controls (2.0,0.9) and (2.2,0.4) .. (m)
      .. controls (3.4,-0.2) and (4.0,-0.8) .. (c);

\filldraw[black] (a) circle (2pt) node[left] {$a$};
\filldraw[black] (b) circle (2pt) node[above] {$b$};
\filldraw[black] (c) circle (2pt) node[below] {$c$};
\filldraw[black] (d) circle (2pt) node[right] {$d$};
\filldraw[black] (m) circle (2pt) node[above right] {$m$};

\node[blue!70!black] at (4.2,0.9) {$\gamma_{a,d}$};
\node[red!70!black] at (3.8,-0.8) {$\gamma_{b,c}$};

\end{tikzpicture}
\caption{Analogy $a : b :: c : d$ holds since the geodesics $\gamma_{a,d}$ from $a$ to $d$ and $\gamma_{b,c}$ from $b$ to $c$ intersect at their midpoint.}
\label{fig:midpoint-analogy}
\end{figure}

In this paper, we propose an alternative characterization, following the idea of \citet{lepage2024any}. 
\begin{proposition}\label{prop:riem_pa}
Let $(\mathcal{M}, g)$ be a geodesically convex Riemannian manifold and let $\nabla$ be a connection on $\mathcal{M}$. The following relation is a strong proportional analogy on $\mathcal{M}$. We have $a : b :: c : d$ if and only if 
\begin{equation}\label{eq:riem_pa}
    \exp_a\left(\frac{1}{2}\log_a(d)\right) = \exp_b\left(\frac{1}{2}\log_b(c)\right)
\end{equation}
The proof is trivial given the following lemma:
\begin{lemma}
    Given $p, q \in \mathcal{M}$, we have $\exp_p\left(\frac{1}{2}\log_p(q)\right) = \exp_q\left(\frac{1}{2}\log_q(p)\right)$
\label{lem:middle-geodesic}
\end{lemma}
\begin{proof}
We denote by $\gamma$ the geodesic from $p$ to $q$ defined as $t \mapsto \exp_p(t\log_p(q))$. Consider the geodesic in the reverse direction: $\tilde{\gamma}(t) = \gamma(1 - t)$. Since $\dot{\tilde{\gamma}}(t) = - \dot{\gamma}(1-t)$, we have for all $t$:
\[
0 = \nabla_{\dot{\tilde{\gamma}}(t)}\dot{\tilde{\gamma}}(t) = \nabla_{- \dot{\gamma}(1-t)} (- \dot{\gamma}(1-t)) = \nabla_{\dot{\gamma}(1-t)} (\dot{\gamma}(1-t))
\]
which shows that $\tilde{\gamma}$ is a geodesic. Since  $\tilde{\gamma}(0) = q$ and $\tilde{\gamma}(1) = p$, we can conclude that $\dot{\tilde{\gamma}}(0) = \log_q(p)$ and that $\tilde{\gamma}(t) = \exp_q(t \log_q(p))$. 
\begin{align*}
    \exp_p\left(\frac{1}{2}\log_p(q)\right) 
    &= 
    \gamma\left(\frac{1}{2}\right) \\
    &= \tilde{\gamma}\left(\frac{1}{2}\right) = \exp_q\left(\frac{1}{2}\log_q(p)\right)
\end{align*}
which proves the result. 
\end{proof}
\end{proposition}

\subsection{Proportional Analogies on Riemannian Manifolds}

In the general case, there is no unicity of a geodesic and the logarithmic map is only a local operator (i.e. $\log_p(q)$ is defined for $q$ in a neighborhood of $p$). A standard example is the 2-dimensional sphere $\mathbb{S}^2$ on which there exist infinitely many geodesics between the two poles. We now propose a generalization of the proportional analogy defined in Equation~\ref{eq:riem_pa}. 

\begin{proposition}
Let $(\mathcal{M}, g)$ be a Riemannian manifold and let $\nabla$ be a connection on $\mathcal{M}$. The following relation is a proportional analogy on $\mathcal{M}$. We have $a : b :: c : d$ if and only if:
\begin{equation}\label{eq:riem_pa_general}
    \exists v \in T_a\mathcal{M}, \, \exists w \in T_b\mathcal{M}; \, \begin{cases}
        c = \exp_b(w) \\
        d = \exp_a(v) \\
        \exp_a\left(\frac{v}{2}\right) = \exp_b\left(\frac{w}{2}\right)
    \end{cases}
\end{equation}
\label{prop:general-form-analogy}
\end{proposition}

In the general case, the analogy defined in Equation~\ref{eq:riem_pa_general} is not a strong analogy: given $a$ and $b$, we have $a : b :: a : b$ but there exists some $x \neq b$ such that $a : b :: a : x$.  

\begin{corollary}
When $(\mathcal{M}, g)$ is geodesically convex, the proportional analogies defined in Equation~\ref{eq:riem_pa} and in Equation~\ref{eq:riem_pa_general} are the same. 
\end{corollary}

This corollary requires that the whole manifold is geodesically convex, which is a strong assumption. We have a weaker result, stating the two proportional analogies coincide only for points where the condition on the unicity of the geodesic is satisfied. 

\begin{corollary}
Let $a, b, c, d \in \mathcal{M}$. If there exists a unique $v \in T_a\mathcal{M}$ and a unique $w \in T_b\mathcal{M}$ such that $c = \exp_b(w)$ and $d = \exp_a(v)$, then $a : b :: c : d$ if and only if $\exp_a\left(\frac{v}{2}\right) = \exp_b\left(\frac{w}{2}\right)$. 
\end{corollary}

From the way we defined the analogy, we can derive a way to solve analogical equations, i.e. to find the values of $x$ such that $a : b :: c : x$.

\begin{corollary}
Let $a, b, c \in \mathcal{M}$. If there exists a unique $v \in T_a\mathcal{M}$ and a unique $w \in T_b\mathcal{M}$ such that $c = \exp_b(w)$ and $\exp_a\left(\frac{v}{2}\right) = \exp_b\left(\frac{w}{2}\right)$, then $a : b :: c : x$ if and only if $x = \exp_a(v)$. This corresponds to $Sol_\mathcal{M}(a, b, c) = \lbrace \exp_a(v) \rbrace$.
\label{corr:solution-construction}
\end{corollary}

\subsection{Isometry Invariance}

\begin{proposition}[Invariance under isometries]
\label{prop:invariance-isometry}
Let $(\mathcal{M},g)$ and $(\mathcal{N},h)$ be Riemannian manifolds, and let $\phi : M \to N$ be a Riemannian isometry. If $a:b::_{(\mathcal{M}, g)} c:d$
holds according to the proportional analogy defined in Equation~\ref{eq:riem_pa_general}, then $\phi(a):\phi(b)::_{(\mathcal{N}, h)} \phi(c):\phi(d)$ also holds.
\end{proposition}

\begin{proof}
Assume that $a:b::_{(\mathcal{M}, g)} c:d$. There exist tangent vectors $v \in T_a\mathcal{M}$ and $w \in T_b \mathcal{M}$  such that $d = \exp_a(v)$, $c = \exp_b(w)$, and the midpoint reached from $a$ along $v$ is the same as the midpoint reached from $b$ along $w$. The geodesic $\gamma_v(t) = \exp_a(tv)$ is such that $\gamma_v(0) = a$ and $\dot{\gamma_v}(0) = v$. $\phi$ being an isometry, it preserves the geodesics, and consequently the curve $\tilde{\gamma_v}(t) = \phi(\gamma_v(t))$ is also a geodesic such that $\tilde{\gamma_v}(0) = \phi(a)$ and $\dot{\tilde{\gamma_v}}(0) = d\phi_a(v)$. Consequently, for every $t, \phi(\exp_a(t v)) = \exp_{\phi(a)}(t \, d \phi_a(v))$. The same result holds for $\gamma_w(t) = \exp_b(t w)$. Consequently, we have $\phi(d) = \exp_{\phi(a)}(d\phi_a(v))$, $\phi(c) = \exp_{\phi(b)}(d\phi_b(w))$ and $\exp_{\phi(a)}(\frac{1}{2}d\phi_a(v)) = \exp_{\phi(b)}(\frac{1}{2}d\phi_b(w))$, which proves that $\phi(a):\phi(b)::_{(\mathcal{N}, h)} \phi(c):\phi(d)$. 
\end{proof}

\subsection{Robust Proportional Analogies}

As mentioned in the introduction, most previous works on proportional analogies focus on Boolean and symbolic domains. When working on continuous domains, we may need the additional continuity property, indicating that the analogical relation can be preserved under small perturbations. In the context of a relation, this continuity property is assessed by the notion of robustness, which interprets as the resistance against noise.

\begin{definition}[Robust proportional analogy]
\label{def:continuous-ap}
Let $(X, \Delta)$ be a metric space. A proportional analogy on $X$ is called \emph{robust} in $(a, b, c, d) \in X^4$ if, for all $\varepsilon > 0$, there exists $(a^\prime, b^\prime, c^\prime, d^\prime) \neq (a, b, c, d)$ such that $\Delta(a, a^\prime) < \varepsilon$, $\Delta(b, b^\prime) < \varepsilon$, $\Delta(c, c^\prime) < \varepsilon$ and $\Delta(d, d^\prime) < \varepsilon$, and $a^\prime : b^\prime :: c^\prime : d^\prime$. 
\end{definition}

A constructive approach to analogies consists in building $d$ given $a, b$ and $c$, ensuring that $a : b :: c : d$. For instance, the arithmetic proportion (Proposition~\ref{prop:arithmetic}) can be seen as constructing $d = c + b - a$. 

\begin{definition}[Function-based proportional analogy] A proportional analogy is called \textbf{function-based} if there exists a function $f:X^3 \rightarrow X$ such that $a : b :: c : d$ if and only if $d = f(a, b, c)$. 
\end{definition}

\begin{proposition}
Let $(\mathcal{M}, g)$ be a complete, uniquely geodesic Riemannian manifold. Then the analogy defined in Proposition~\ref{prop:riem_pa} is a function-based proportional analogy.
\label{prop:riemfunction-based}
\end{proposition}

\begin{proof}
We simply take 
\begin{equation}
    f(a,b,c) = \exp_a\left(2 \log_a\left( \exp_b\left( \frac{1}{2} \log_b(c) \right) \right) \right)
\end{equation}
\end{proof}

\begin{proposition}
Let $f: X^3 \rightarrow X$ define a function-based proportional analogy. If $f$ is continuous at $(a, b, c)$ and $(a,b,c)$ is not isolated
in $X^3$, then the analogy is robust in $(a, b, c, f(a, b, c))$. 
\label{prop:continuity-implies-robustness}
\end{proposition}

\begin{proof}
Assume that $f$ is continuous in $(a, b, c)$. Let $\varepsilon > 0$. There exists $\delta_a, \delta_b, \delta_c > 0$ such that if $\Delta(a, a^\prime) < \delta_a$, $\Delta(b, b^\prime) < \delta_b$ and $\Delta(c, c^\prime) < \delta_c$, then $\Delta(f(a, b, c), f(a^\prime, b^\prime, c^\prime)) < \varepsilon$. In particular, we can find $a^\prime$ such that $\Delta(a, a^\prime) < \min \lbrace \delta_a, \varepsilon \rbrace \leq \varepsilon$, $b^\prime$ such that $\Delta(b, b^\prime) < \min \lbrace \delta_b, \varepsilon \rbrace \leq \varepsilon$ and $c^\prime$ such that $\Delta(c, c^\prime) < \min \lbrace \delta_c, \varepsilon \rbrace \leq \varepsilon$ guaranteeing that $d^\prime = f(a^\prime, b^\prime, c^\prime)$ is such that $\Delta(d, d^\prime) < \varepsilon$. By definition of function-based analogies, we have $a^\prime : b^\prime :: c^\prime : d^\prime$, which shows the robustness.   
\end{proof}

Combining the results of Propositions \ref{prop:riemfunction-based} and \ref{prop:continuity-implies-robustness}, and since the $\exp$ maps is a diffeomorphism on a Hadamard manifold, we obtain the following important result:
\begin{corollary}
On a Hadamard manifold of positive dimension, the analogy of Proposition~\ref{prop:riem_pa} is robust.
\label{corr:robustness-geodesically-convex}
\end{corollary}


%% file: sections/real-line.tex
We begin by investigating analogies in the Euclidean space $\mathbb{R}^n$, where the geometry remains sufficiently simple to derive several general structural results. In particular, this setting allows us to characterize the realizability of analogies and highlights the role of the underlying Riemannian metric.

\subsection{Analogies on $\mathbb{R}$}

We first consider the simple case where $\mathcal{M} = \mathbb{R}$. In that case, the midpoint has a very simple characterization:
\begin{lemma}
For any $x_0 \in \mathbb{R}$, the function $\phi:x \mapsto  \int_{x_0}^x \sqrt{g(t)} \, dt$ is a diffeomorphism, and the midpoint between $x$ and $y$ with metric $g$ is:
\begin{equation}
    m_g(x, y) = \phi^{-1}\left( \frac{\phi(x) + \phi(y)}{2} \right).
    \label{eq:midpoint-realline}
\end{equation}
\end{lemma}
\begin{proof}
$\phi$ is differentiable with $\phi^\prime(x) > 0$ for all $x$, so $\phi$ is a strictly increasing diffeomorphism from $\mathbb{R}$ to its image. Assuming $x < y$, the midpoint $m = m_g(x, y)$ is defined by:
\[
\int_x^m \sqrt{g(t)} \, dt = \int_m^y \sqrt{g(t)} \, dt 
\]
which is equivalent to $\phi(m) - \phi(x) = \phi(y) - \phi(m)$. Since the image $\phi(\mathbb{R})$ of $\phi$ is a segment, $\frac{1}{2}(\phi(x) + \phi(y)) \in \phi(\mathbb{R})$, and $m = \phi^{-1}\left( \frac{\phi(x) + \phi(y)}{2} \right).$
\end{proof}

\begin{theorem}
Let $a, b, c, d \in \mathbb{R}$ be pairwise distinct points, and $I_{x,y} = (\min\lbrace x, y \rbrace, \max\lbrace x, y \rbrace)$ the open interval between $x$ and $y$. There exists a smooth Riemannian metric $g$ on $\mathbb{R}$ such that the analogy $a : b ::_g c : d$ holds, if and only if $I_{b, c} \subset I_{a, d}$ or $I_{a, d} \subset I_{b, c}$. 
\end{theorem}

\begin{proof}
Let $x_1<x_2<x_3<x_4$ be the four points $a,b,c,d$ arranged in increasing order. 

We first prove necessity. Suppose that $a : b ::_g c : d$ but that the intervals $I_{a,d}$ and $I_{b,c}$ are not nested. 
There are two possible pairings, up to exchanging the two pairs. 

The first possibility is  $\{a,d\}=\{x_1,x_2\}$ and $\{b,c\}=\{x_3,x_4\}$.  Since $\phi$ is strictly increasing, $\phi(x_1) + \phi(x_2) < \phi(x_3) + \phi(x_4)$ and $\phi(a) + \phi(d) \neq \phi(b) + \phi(c).$ The second possibility is $\{a, d\}=\{x_1, x_3\}$ and $\{b, c\} = \{x_2, x_4\}$.  Again, strict monotonicity gives $\phi(x_1)< \phi(x_2)$ and  $\phi(x_3)<\phi(x_4)$. Adding these inequalities yields  $\phi(x_1) + \phi(x_3) < \phi(x_2)+\phi(x_4)$, and hence  $\phi(a) + \phi(d) \neq \phi(b) + \phi(c)$.  Thus, if the two intervals are not nested, no smooth Riemannian metric can satisfy $a:b::_g c:d$. 

We now prove sufficiency. Suppose that the intervals are nested. Since the four points are pairwise distinct, the larger interval must have endpoints $x_1$ and $x_4$, while the smaller interval has endpoints $x_2$ and $x_3$. Therefore, $\{a, d\} = \{x_1, x_4\}$ and $\{b, c\} = \{x_2, x_3\}$ or conversely. It is enough to construct a smooth positive function $\rho$ such that $\int_{x_1}^{x_2}\rho(t)\,dt = \int_{x_3}^{x_4}\rho(t)\,dt$. 

Suppose first that $x_2-x_1\leq x_4-x_3$. Let $\psi\in C_c^\infty((x_1,x_2))$ be a non-negative smooth function such that $\int_{x_1}^{x_2} \psi(t) \, dt = 1$ and $\psi(t) = 0$ outside $[x_1, x_2]$. We define $\rho(t) = 1 + ((x_4 - x_3) - (x_2 - x_1)) \psi(t)$. We can check that $\int_{x_1}^{x_2} \rho(t) \, dt = \int_{x_3}^{x_4} \rho(t) \, dt$. We can build a function $\rho$ similarly in the case where $x_2 - x_1 > x_4 - x_3$. In both cases, $\rho$ is smooth, strictly positive, and can therefore be used to define the metric $g(t) = \rho(t)^2$ satisfying $m_g(a, d) = m_g(b, c)$.
\end{proof}

The generalized means considered by \citet{lepage2024any} can be interpreted as Riemannian midpoints on the positive real line. Indeed, the power mean of order $p \neq 0$ is induced by the metric $g_p(x) = p^2 x^{2 p - 2}$, while the geometric mean is induce by the metric $g_0(x) = x^{-2}$. Their existence and uniqueness theorem for $0 < a < b < c < d$ therefore establishes that, for every nested quadruple of positive points, there exists a unique metric within this particular one-parameter family that realizes the analogy. Our result characterizes existence over the full class of smooth Riemannian metrics and on the whole real line, and shows that nestedness is also necessary.

\subsection{Analogies on $\mathbb{R}^n$ for $n \geq 2$}

The case $n = 1$ presented the difficulty that the intervals had to be nested in order to have the correct properties of a metric. In the general case, there is no global order anymore, which removes the constraints on the geodesics. 

\begin{theorem}
Let $a, b, c, d \in \mathbb{R}^n$ pairwise distinct points of $\mathbb{R}^n$, with $n \geq 2$. There exists a metric $g$ such that $a : b ::_g c : d$. 
\end{theorem}
\begin{proof}
Consider four pairwise distinct points $A, B, C, D \in \mathbb{R}^n$ such that $A + D  = B + C$. The group $\text{Diff}(\mathbb{R}^n)$ acts $k$-transitively on $\mathbb{R}^n$ for every finite $k$ when $n \geq 2$~\cite{michor1994ntransitivity}, and in particular acts $4$-transitively. Consequently, there exists a diffeomorphism $\phi: \mathbb{R}^n \rightarrow \mathbb{R}^n$ such that $\phi(a) = A$, $\phi(b) = B$, $\phi(c) = C$ and $\phi(d) = D$. 

Let $g_E$ be the Euclidean metric on $\mathbb{R}^n$. The pullback metric $g = \phi^* g_E$ on $\mathbb{R}^n$ is defined as $g_x(u, v) = g_{E,\phi(x)}(D \phi_x(u), D\phi_x(v))$. The function $\phi$ is then an isometry from $(\mathbb{R}^n, g)$ onto $(\mathbb{R}^n, g_E)$. Since $A : B ::_{g_E} C : D$, Proposition~\ref{prop:invariance-isometry} guarantees that $a : b ::_g c : d$. 
\end{proof}

The intrinsic difference between the cases $n = 1$ and $n \geq 2$ lies in the possibility to build a diffeomorphism, indicating that the realizability of analogies depends on the transitivity properties of the diffeomorphism group of the underlying manifold.

%% file: sections/symmetric-spaces.tex
Riemannian symmetric spaces, introduced by \citet{cartan1926classe}, constitute one of the central classes of Riemannian manifolds and play a fundamental role in differential geometry and Lie theory. They include Euclidean spaces, spheres, hyperbolic spaces, Grassmann manifolds and the manifold of symmetric positive-definite matrices \cite{helgason1979differential}.

\begin{definition}[Riemannian symmetric space]
A connected Riemannian manifold $(\mathcal{M}, g)$ is symmetric if for every point $p \in \mathcal{M}$, there exists an isometry $s_p: \mathcal{M} \rightarrow \mathcal{M}$ such that $s_p(p) = p$ and $D(s_p)_p = - \text{Id}$. 
\end{definition}

From this definition, it follows that, at every point $p$ of a Riemannian symmetric space, any geodesic will be reflected, in the sense that $\gamma(t)$ is transformed into $\gamma(-t)$. This property induces an important characterization of analogies.

\begin{theorem}
Let $(\mathcal{M}, g)$ be a Riemanian symmetric space. For all $a, b, c, d \in \mathcal{M}$, we have $a : b ::_g c : d$ if and only if there exists $m \in \text{Mid}(b, c)$ such that $d = s_m(a)$, where:
\begin{equation}
    \text{Mid}(b, c) = \left \lbrace \exp_b\left( \frac{w}{2}\right) : w \in T_bM, \exp_b(w) = c \right\rbrace.
\end{equation}
\label{thm:characterization-symmetric-spaces}
\end{theorem}

\begin{proof}
We first prove the forward implication. Assume that $a : b ::_g c : d$. Then there exists $v \in T_a\mathcal{M}$ and $w \in T_b\mathcal{M}$ satisfying the conditions presented in Proposition~\ref{prop:general-form-analogy}. Define $m = \exp_a(v/2) = \exp_b(w/2)$. In particular, we observe that $m \in \text{Mid}(b, c)$. Consider the geodesic $\gamma(t) = \exp_a(tv)$ for $t \in \mathbb{R}$. It satifies $\gamma(0) = a$, $\gamma(1/2) = m$ and $\gamma(1) = d$. We reparametrize this geodesic as $\tilde{\gamma}$ defined by $\tilde{\gamma}(t) = \gamma(t + 1/2)$. By construction, $\tilde{\gamma}$ is a geodesic through $m$ such that $\tilde{\gamma}(0) = m$, $\tilde{\gamma}(-1/2) = a$ and $\tilde{\gamma}(1/2) = d$. By definition of the geodesic symmetry, we have $s_m(\tilde{\gamma(t)}) = \tilde{\gamma}(-t)$ for all $t$. In particular for $t = -1/2$, we obtain $d = \tilde{\gamma}(1/2) = s_m(\tilde{\gamma}(-1/2)) = s_m(a)$.  

For the reverse implication, we consider $m \in \text{Mid}(b,c)$ such that $d = s_m(a)$.
Since a Riemannian symmetric space is geodesically complete, there exists $v \in T_m\mathcal{M}$ such that $a = \exp_m(v)$. By definition of the symmetry, we have $s_m(\exp_m(v)) = \exp_m(-v)$. Hence $d = s_m(a) = \exp_m(-v)$. 
Consider the curve $\gamma(t) = \exp_m((1 - 2t)v)$ for $t \in \mathbb{R}$. It is an affinely parametrized geodesic, and it satisfies $\gamma(0) = a$, $\gamma(1/2) = m$ and $\gamma(1) = d$. Denoting $u = \dot{\gamma}(0) \in T_a\mathcal{M}$, $\gamma$ is the geodesic starting at $a$ and with initial velocity $u$, so $d = \exp_a(u)$ and $m = \exp_a(u/2)$. Since $m = \exp_b(w/2)$ and $c = \exp_b(w)$, we have $a : b ::_g c : d$ by Proposition~\ref{prop:general-form-analogy}.
\end{proof}

A direct application of this result is the Euclidean case. We can observe that $(\mathbb{R}^n, g_E)$ is a Riemanian symmetric space with $s_m(a) = 2 m - a$. When $m$ is the midpoint between $b$ and $c$, we have $m = \frac{b + c}{2}$, and we retrieve the parallelogram rule $d = c + b - a$. 

\begin{proposition}[Robustness on Riemannian symmetric spaces]
Let $(\mathcal{M},g)$ be a Riemannian symmetric space of positive dimension. The proportional analogy defined in Theorem~\ref{thm:characterization-symmetric-spaces} is robust at every $(a,b,c,d)\in \mathcal{M}^4$ such that $a:b::_{(\mathcal{M},g)}c:d$:
\end{proposition}

\begin{proof}
Assume that $a:b::_{(\mathcal{M},g)}c:d$. By Theorem~\ref{thm:characterization-symmetric-spaces}, there exists $m\in\operatorname{Mid}(b,c)$ such that $d=s_m(a)$.

Let $\varepsilon>0$. Since $\mathcal{M}$ has positive dimension, there exists $a'\neq a$ such that $\operatorname{dist}(a,a')<\varepsilon$. Set
\[
b'=b,\qquad c'=c,\qquad d'=s_m(a').
\]
Since $m\in\operatorname{Mid}(b',c')$, Theorem~\ref{thm:characterization-symmetric-spaces}
implies $a':b'::_{(\mathcal{M},g)}c':d'$. Moreover, $s_m$ is a Riemannian isometry, and hence
\[
\operatorname{dist}(d,d') = \operatorname{dist}\bigl(s_m(a),s_m(a')\bigr) = \operatorname{dist}(a,a') < \varepsilon.
\]
The other two distances are zero:
\[
\operatorname{dist}(b,b')=\operatorname{dist}(c,c')=0.
\]
Finally, $(a',b',c',d')\neq(a,b,c,d)$ because $a'\neq a$.
Therefore, the analogy is robust at $(a,b,c,d)$.
\end{proof}

This result shows that valid analogical quadruples are not isolated: every neighborhood of a valid quadruple contains another valid quadruple. Corollary~\ref{corr:robustness-geodesically-convex} does not provide a strong enough characterization here: the proposition states that the space does not need to be a Hadamard manifold to ensure robustness of the analogy when it is symmetric.

\subsection{Analogies on Spheres}

The sphere $\mathbb{S}^n \subset \mathbb{R}^{n+1}$ is defined as the set $\mathbb{S}^n = \lbrace x \in \mathbb{R}^{n+1} \, ; \, \|x\|_E^2 = 1 \rbrace$. It can be verified that the sphere $\mathbb{S}^n$ is a Riemannian symmetric space, with $s_m(x) = 2 \langle m, x \rangle m - x$, where $\langle., .\rangle$ designates the Euclidean scalar product in the ambient space $\mathbb{R}^{n+1}$~\cite{helgason1979differential}.

\begin{corollary}
The solutions of the analogical equation $a : b :: c : x$ on the sphere $\mathbb{S}^n$ are:
\begin{equation}
\Sol{\mathbb{S}^n} = \big\lbrace 2 \langle m, a\rangle m - a \, : \, m \in \Mid{b}{c} \big\rbrace
\end{equation}
\label{corr:analogies-sphere}
\end{corollary}

When $b$ and $c$ are not antipodal (i.e. when $c \neq -b$), the set of midpoints contains two elements: $\Mid{b}{c} = \left\lbrace \frac{b + c}{\|b+c\|}, -\frac{b + c}{\|b+c\|} \right\rbrace$. Both midpoints induce the same geodesic symmetry, and the analogical equation therefore admits the unique solution $d = \frac{\langle a, b + c \rangle}{1 + \langle b, c \rangle} (b + c) - a$.

When $b = -c$, the midpoint is not unique anymore: $\Mid{b}{-b} = \lbrace m \in \mathbb{S}^n \, : \, \langle m, b \rangle = 0 \rbrace$, and $\Sol{\mathbb{S}^n} = \lbrace d \in \mathbb{S}^n \, : \, \langle d, b \rangle = - \langle a, b \rangle \rbrace$. In particular, in the case where $n = 2$, this corresponds to the points of azimuthal angle $\theta = \pi - \theta_a$, with $\theta_a$ the azimuthal angle of $a$.  

The sphere $\mathbb{S}^n_R$ of radius $R$ has symmetry $s_m(x) = \frac{2}{R^2} \langle m, x \rangle m - x$. 

\begin{corollary}
The solutions of the analogical equation $a : b :: c : x$ on the sphere $\mathbb{S}^n$ are:
\begin{equation}
\Sol{\mathbb{S}^n_R} = \left\lbrace \frac{2}{R^2} \langle m, a\rangle m - a \, : \, m \in \Mid{b}{c} \right\rbrace
\end{equation}
\label{corr:analogies-sphere-radius}
\end{corollary}

\subsection{Analogies on Hyperbolic Spaces}

Define the $n$-dimensional hyperboloid model as 
\[
\mathbb{H}^n = \lbrace x \in \mathbb{R}^{n+1} \, : \, \langle x, x \rangle_L = -1, \, x_0 > 0 \rbrace
\]
where $\langle ., . \rangle_L$ is the Lorentzian inner product defined as $\langle x, y \rangle_L = -x_0 y_0 + \sum_{i=1}^n x_i y_i$. The Riemannian metric on $\mathbb{H}^n$ is induced from the Lorentzian metric. Since all standard models of hyperbolic space are mutually isometric, working on $\mathbb{H}^n$ entails no loss of generality. We refer the reader to \citet{ratcliffe2006foundations} for more details on hyperbolic spaces.

It can be verified that the hyperboloid $\mathbb{H}^n$ is a Riemannian symmetric space with $s_m(x) = - 2 \langle x, m \rangle_L m  - x$. 

Unlike the sphere, hyperbolic space is a Hadamard manifold. Consequently, any two points are joined by a unique geodesic, and therefore admit a unique geodesic midpoint. For any $b, c \in \mathbb{H}^n$, the unique midpoint is $m = \frac{b + c}{\sqrt{2 - 2 \langle b, c \rangle_L}}$. 

\begin{corollary}
The analogical equation $a : b :: c : x$ on the hyperboloid $\mathbb{H}^n$ admits a unique solution:
\begin{equation}
\Sol{\mathbb{H}^n} = \left\lbrace - a + \frac{\langle a, b+c \rangle_L}{1 - \langle b, c \rangle_L} (b + c) \right\rbrace
\end{equation}
\end{corollary}

\subsection{Analogies on  Symmetric Positive Definite Matrices}

A matrix $M$ is \emph{symmetric positive definite} if and only if $M^T = M$ and $x^T M x > 0$ for all $x \in \mathbb{R}^n \setminus \lbrace 0 \rbrace$. We denote by $SPD_n$ the space of symmetric positive definite matrices of size $n$. Symmetric positive definite matrices play an important role in various applications of machine learning such as metric learning~\cite{zadeh2016geometric} or sparse coding~\cite{cherian2016riemannian}. 

The tangent space $T_{\Sigma}SPD_n$ is the space of symmetric matrices. An important family of metrics for $SPD_n$ is the affine-invariant metrics of the form $g_{\Sigma}^{AI}(V,W) = \alpha Tr(\Sigma^{-1} V \Sigma^{-1} W) + \beta Tr(\Sigma^{-1} V) Tr(\Sigma^{-1} W)$, with $\alpha > 0$ and $\alpha + n \beta > 0$. \citet{thanwerdas2019affine} show that the manifold $(SPD_n, g^{AI})$ is a Riemannian symmetric space with symmetry $s_{\Sigma}(\Lambda) = \Sigma \Lambda^{-1} \Sigma$. The midpoints between $B$ and $C$ are characterized by $C = s_M(B) = M B^{-1} M$. The equation has a unique solution $M = B \# C =  B^{1/2} (B^{-1/2} C B^{-1/2})^{1/2} B^{1/2}$. 

\begin{corollary}
The analogical equation $A : B :: C : X$ on the space of symmetric positive definite matrices $SPD_n$ admits a unique solution:
\begin{equation}
\Sol{SPD_n} = \left\lbrace (B \# C) A^{-1} (B \# C) \right\rbrace
\end{equation}
\label{corr:spd}
\end{corollary}

%% file: sections/shape.tex
In this section, we explore how the introduced analogies apply to shape spaces. We will start with the simple case of Kendall's space of triangles~\cite{kendall1984shape}, which reduces to the case of the sphere $\mathbb{S}^2$, and will then extend to 3D shapes, with potential applications in computer graphics.

\subsection{A Short Introduction to Shape Spaces}

Many tasks in computer vision, medical imaging and computer graphics rely on comparing geometric objects rather than individual pixels. Examples include object recognition, motion analysis and statistical shape analysis. In these settings, the quantity of interest is the shape of an object, independently of transformations such as translation, rotation or scaling.

A natural way to formalize this idea is to represent an object by a finite set of landmarks and to identify configurations that differ only by these nuisance transformations. The resulting quotient space, called a shape space, is endowed with a Riemannian structure, allowing distances, geodesics and statistical quantities such as means to be defined intrinsically~\cite{kendall1984shape}.

Shapes can also be modeled as continuous curves or surfaces~\cite{srivastava2016functional}. In this setting, a shape is represented by an embedding of a reference manifold into the ambient space, and equivalent parameterizations are identified through the action of the diffeomorphism group. The resulting quotient spaces inherit a rich differential geometric structure. 

\subsection{Space of Triangles}

A landmark configuration consisting of three points in the plane can be interpreted as a triangle. Since the shape of a triangle should be independent of its position, orientation and size, the construction of~\citet{kendall1984shape} first removes translations by centering the landmarks, then normalizes the configuration to unit size, and finally identifies configurations that differ only by a planar rotation.

Let $(z_1, z_2, z_3) \in \mathbb{C}^3$ denote the coordinates of the three vertices of the triangle, interpreted as complex numbers. Quotienting by the translation and scale invariance results in the two-dimensional preshape space $\lbrace \mathbf{z} \in \mathbb{C}^3 \, : \, \sum_{i=1}^3 z_i = 0, \| \mathbf{z} \| = 1 \rbrace$. Quotienting this space by the action of planar rotations  $z \sim e^{i\theta} z$ identifies preshapes representing the same triangle. The resulting quotient is Kendall's shape space of planar triangles, which is isometric to the sphere~$\mathbb{S}^2$. Consequently, the result established in Corollary~\ref{corr:analogies-sphere} for the sphere directly applies to planar triangles.

\subsection{Three-Dimensional Shapes}

As an illustration of the generalization capacity of our method, we consider analogical reasoning on three-dimensional triangle meshes. This application is closely related to the classical problem of deformation transfer~\cite{sumner2004deformation}, where a deformation observed on one shape is transferred to another. Our objective is not to compete with specialized deformation transfer methods, but rather to demonstrate that our general analogy relation naturally extends to a challenging geometric domain without requiring an application-specific formulation.

For the geometric operations on shape spaces, we rely on the \texttt{Morphomatics} library~\cite{Morphomatics}. The library represents triangular meshes through fundamental coordinates, providing a Riemannian manifold structure on which we directly apply the analogy solver of Proposition~\ref{prop:riemfunction-based}. Given three meshes $A$, $B$, and $C$, the missing mesh $D$ is obtained by solving the corresponding analogical equation.

The experiments are conducted on animal meshes derived from the SMALR model~\cite{zuffi2018lions}. Since all meshes are generated from the same template, they share an identical triangulation and establish a one-to-one correspondence between vertices. This common topology satisfies the assumptions of the fundamental-coordinate representation; otherwise, a preprocessing step would be required to establish vertex correspondences.

Figure~\ref{fig:placeholder} (domain 3) provides a typical example. Meshes $A$ and $B$ represent the same dog in a reference pose and a transformed pose, while mesh $C$ represents a cow in the reference pose. The computed mesh $D$ transfers the deformation $A \rightarrow B$ to the cow while preserving its geometric characteristics. Although not competitive with dedicated deformation transfer methods, the successful transfer and the low execution time (approximately $2\,\mathrm{s}$) illustrate the ability of the proposed framework to generalize beyond the domains for which it was originally designed.

%% file: sections/categorical-distributions.tex
The space of $n$-dimensional categorical distributions is defined as the simplex:
\begin{equation}
    \Delta^{n-1} = \left\lbrace p \in (0, 1)^n \mid p_i > 0 \quad \text{and} \quad \sum_{i=1}^n p_i = 1 \right\rbrace
\end{equation}

Several geometries have been proposed on the simplex of categorical distributions. A comprehensive comparison is given by \citet{nielsen2019clustering}. In this section, we consider two Riemannian geometries: the Fisher-Rao metric~\cite{nielsen2020elementary} and the Aitchison metric~\cite{egozcue2006hilbert}. Other geometries, such as the Hilbert geometry, are not Riemannian and do not enter the scope of this paper. Future works should investigate how our framework can be extended to such geometries.

\subsection{Fisher-Rao Metric}

Information Geometry~\cite{nielsen2020elementary} introduces Fisher-Rao metric as a natural metric choice for the manifolds of parametric probability distributions. In the specific case of categorical distribution, the Fisher-Rao metric is defined as:
\begin{equation}
    g_{\text{FR},P}(v_1, v_2) = \sum_{i=1}^n \frac{v_1^i v_2^i}{p_i}.
\end{equation}
The square-root mapping $\phi: \Delta^{n-1} \rightarrow \mathbb{S}_{2,+}^{n-1}$ defined as $\phi(p) = (2 \sqrt{p_1}, \dotsc, 2 \sqrt{p_n})$ is an isometry from the Fisher-Rao simplex to the positive orthant of the sphere of radius $2$, denoted $\mathbb{S}_{2,+}^{n-1}$.  Using the result of Proposition~\ref{prop:invariance-isometry}, we can show the following result:

\begin{proposition}
Let $a, b, c \in \Delta^{n-1}$. Define
\begin{equation}
    r = 2 \frac{\langle\sqrt{a}, \sqrt{b} + \sqrt{c} \rangle}{\|\sqrt{b} + \sqrt{c}\|^2} (\sqrt{b} + \sqrt{c}) - \sqrt{a}.
\end{equation}
The analogical equation $a : b ::_{(\Delta^{n-1}, g_{\text{FR}})} c :x$ admits a solution if and only if $r_i > 0$ for all $i$. When the condition holds, the solution is unique and is given by $d = (r_1^2, \dotsc, r_n^2)$. 
\label{prop:fisher-rao}
\end{proposition}
\begin{proof}
Let $A = \phi(a)$, $B = \phi(b)$ and $C = \phi(c)$. Since $B, C \in \mathbf{S}_{2,+}^{n-1}$, the two points cannot be antipodal. Therefore, they admit a single midpoint $M = 2 \frac{\sqrt{b} + \sqrt{c}}{\|\sqrt{b} + \sqrt{c}\|}$. We observe that $M \in \mathbf{S}_{2,+}^{n-1}$, so $A$ and $M$ are not antipodal. By Corollary~\ref{corr:analogies-sphere-radius}, the solution of the analogical equation on $\mathbb{S}_2^{n-1}$ is $D = 2r$. It is a solution on $\mathbf{S}_{2,+}^{n-1}$ if and only if each $r_i > 0$. The solution on $\Delta^{n-1}$ is obtained by applying $\phi^{-1}$. 
\end{proof}

The fact that all analogical equations do not necessarily have a solution with Fisher-Rao metric can be explained by the fact that $(\Delta^{n-1}, g_{\text{FR}})$ is not geodesically complete: for a given $p \in \Delta^{n-1}$ and $u \in T_p\Delta^{n-1}$, the function $t \mapsto \exp_p(tu)$ is not defined for all $t \in \mathbb{R}$. Consequently, geodesics may leave the manifold in finite time, preventing some geodesic symmetries from being defined. This explains why analogical equations are not always solvable.

\subsection{Aitchison Metric}

Another approach consists in finding a diffeomorphism $\psi: \Delta^{n-1} \rightarrow \mathbb{R}^{n-1}$ to endow $\Delta^{n-1}$ with a pullback of the Euclidean metric. The solution proposed by \citet{aitchison1982statistical} consists in introducing the centered log ratio $\text{clr}(p) = \left(\log \frac{p_1}{f(p)}, \dotsc, \log \frac{p_n}{f(p)} \right)$ with $f(p) = \left( \prod_{i=1}^n p_i \right)^{1/n}$ the geometric mean of the vector $p$. The image of the clr mapping is the hyperplane $H = \left\lbrace x \in \mathbb{R}^n \, : \, \sum_{i=1}^n x_i = 0 \right\rbrace$, and therefore defines a diffeomorphism $\psi: \Delta^{n-1} \rightarrow \mathbb{R}^{n-1}$. We define $g_A = \psi^*g_E$ as the pullback of the Euclidean metric $g_E$ on $\mathbb{R}^{n-1}$. Consequently, $(\Delta^{n-1}, g_A)$ is isometric to $(\mathbb{R}^{n-1}, g_E)$, and we can apply Proposition~\ref{prop:invariance-isometry}. 

\begin{proposition}
Let $a, b, c \in \Delta^{n-1}$. The analogical equation $a : b ::_{(\Delta^{n-1}, g_A)}c : x$ has a unique solution:
\begin{equation}
    d = \left(\frac{c_1 b_1 / a_1}{\sum_{i=1}^n c_i b_i / a_i}, \dotsc, \frac{c_n b_n / a_n}{\sum_{i=1}^n c_i b_i / a_i} \right)
    \label{eqn:solution-categorical-aitchison}
\end{equation}
\label{prop:aitchison}
\end{proposition}
\begin{proof}
It follows directly from Proposition~\ref{prop:invariance-isometry} that the analogy $a : b ::_{(\Delta^{n-1}, g_A)}c : d$ holds if and only if $\text{clr}(a) : \text{clr}(b) ::_{(H, g_E)} \text{clr}(c) : \text{clr}(d)$. Consequently, the equation has a unique solution satisfying for each component $i$:
\[
d_i = \frac{\exp(\text{clr}(c)_i + \text{clr}(b)_i - \text{clr}(a)_i)}{\sum_{j=1}^n \exp(\text{clr}(c)_j + \text{clr}(b)_j - \text{clr}(a)_j)}
\]
which can be developed into the solution of Equation~\ref{eqn:solution-categorical-aitchison}.
\end{proof}

\subsection{Application}

To evaluate the proposed analogical framework on categorical distributions, we consider a preference transfer task based on the MovieLens 1M dataset. 

\paragraph{Task description.}
For each demographic cohort $g$ (e.g. age, gender, occupation) and movie category $c$ present in the dataset, we compute the empirical distribution of ratings $p_{gc} = (p_{gc}(1), \dotsc, p_{gc}(5))$ where $p_{gc}(r)$ is the probability that a member of the cohort $g$ gives rating $r$ to a movie of category $c$. To avoid zero probabilities, each distribution is estimated using symmetric Dirichlet smoothing.

\begin{figure}
    \centering
    \includegraphics[width=\linewidth]{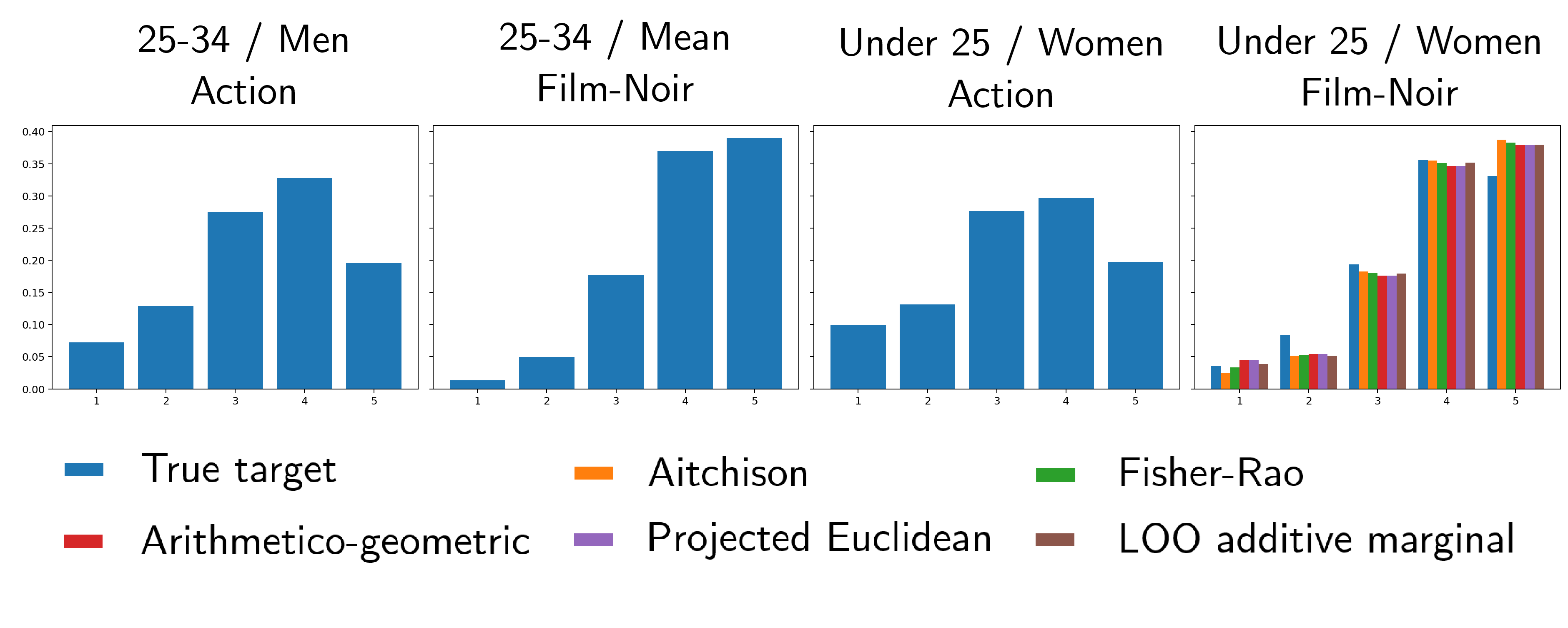}
    \caption{An example of a rating profile transfer between two cohorts ($25-34$ men and $< 25$ women) and two categories (action and film noir). A rating profile is represented as a distribution $p \in \Delta^4$.}
    \label{fig:example}
\end{figure}

The objective is to predict the rating distribution of a target cohort on a target category from three observed distributions. More precisely, given a reference cohort $g_1$, a target cohort $g_2$, a reference category $c_1$ and a target category $c_2$, we consider the analogical equation $p_{g_1, c_1} : p_{g_1, c_2} :: p_{g_2, c_1} : x$
whose solution is used as an estimate of the unknown distribution $p_{g_2, c_2}$ (Figure~\ref{fig:example}). Intuitively, the transformation induced by changing from category $c_1$ to $c_2$ for the reference cohort is transferred to the target cohort.

\begin{table*}[t]
\centering
\caption{Mean Jensen-Shannon divergence ($\times 10^{-3}$) on the MovieLens preference transfer task.
Lower is better. Values in brackets denote 95\% bootstrap confidence intervals.}
\label{tab:movielens-js}
\small
\setlength{\tabcolsep}{3.5pt}
\begin{tabular}{l|cc|cccc}
Experiment & LOO Genre & LOO Add. & Fisher-Rao & Aitchison & Prade \\
\hline
\hline

Age & $1.323_{[1.108,\,1.560]}$ & $0.366_{[0.259,\,0.487]}$ & $\mathbf{0.274}_{[0.190,\,0.373]}$ & $0.314_{[0.210,\,0.451]}$ & $0.294_{[0.208,\,0.393]}$ \\

Age + Gender & $1.731_{[1.449,\,2.027]}$ & $0.842_{[0.641,\,1.064]}$ & $\mathbf{0.679}_{[0.523,\,0.863]}$ & $0.708_{[0.546,\,0.881]}$ & $0.686_{[0.531,\,0.863]}$ \\

Occupation
& $1.962_{[1.690,\,2.244]}$ & $1.134_{[0.919,\,1.408]}$ & $0.956_{[0.786,\,1.164]}$ & $\mathbf{0.931}_{[0.766,\,1.128]}$ & $0.990_{[0.822,\,1.181]}$ \\

Single Genre
&
$4.360_{[3.538,\,5.216]}$
&
$3.432_{[2.498,\,4.527]}$
&
$\mathbf{2.649}_{[1.984,\,3.411]}$
&
$2.865_{[2.152,\,3.655]}$
&
$2.677_{[1.961,\,3.462]}$
\\

Independent Users
&
$2.431_{[2.079,\,2.823]}$
&
$2.040_{[1.667,\,2.464]}$
&
$\mathbf{1.899}_{[1.554,\,2.291]}$
&
$1.949_{[1.586,\,2.355]}$
&
$1.904_{[1.568,\,2.272]}$
\\

Occupation + Genre + Decade
&
$7.378_{[6.993,\,7.779]}$
&
$8.138_{[7.701,\,8.585]}$
&
$\mathbf{7.350}_{[6.984,\,7.738]}$
&
$7.383_{[7.003,\,7.767]}$
&
$7.368_{[6.981,\,7.745]}$
\end{tabular}
\end{table*}
\paragraph{Models and baselines.}

We compare three analogical models: the Fisher-Rao (Proposition~\ref{prop:fisher-rao}), Aitchison (Proposition~\ref{prop:aitchison}) and the arithmetico-geometric analogy of~\citet{prade2025analogical}. 
As baselines, we consider two leave-one-out predictors: the average distribution of the target category, and an additive model combining cohort and category marginal distributions.

\paragraph{Metrics.}
Performance is evaluated using the Jensen Shannon divergence. Statistical significance is assessed through paired permutation tests over all hidden cohort-category pairs. Bootstrap confidence intervals are reported for all average performance measures.

\paragraph{Conditions.}
We evaluate the proposed models under six experimental conditions of increasing difficulty. \textbf{(1) Age:} Users are partitioned according to their age group, and movie categories correspond to the standard MovieLens genres. \textbf{(2) Age + Gender.} Cohorts are refined by combining age group and gender. \textbf{(3) Occupation.} Cohorts are defined by the users' occupation. Since MovieLens contains a larger number of occupations than age groups, this setting increases the diversity of user profiles and the complexity of the transfer task. \textbf{(4) Single Genre.} Users are partitioned into age groups. The evaluation is restricted to movies associated with a single genre, reducing taste ambiguity. \textbf{(5) Independent Users.} Users are randomly partitioned into disjoint training and evaluation subsets. This prevents the same users from contributing to both the prediction and the ground truth, thereby providing a stricter evaluation of the analogical transfer. \textbf{(6) Occupation + Genre + Decade.} The most challenging setting combines occupation-based cohorts with movie categories defined by genre and release decade, producing a much finer partition of the data and reducing the number of observations available for each distribution.

\paragraph{Results.}

Table~\ref{tab:movielens-js} reports the prediction performance on the six experimental settings. Across all experiments, the analogical approaches globally outperform both leave-one-out baselines. In particular, the gain over the genre-only baseline is substantial in the standard demographic transfer tasks, showing that the proposed analogies successfully transfer rating profiles between different user cohorts rather than relying solely on the target movie category.

Among the proposed geometries, Fisher-Rao achieves the lowest average Jensen-Shannon divergence in five of the six experiments, while the Aitchison geometry performs best on the occupation-based setting. The arithmetico-geometric analogy remains consistently close to the two Riemannian approaches. The relatively small differences between the analogical models suggest that the proposed framework is robust with respect to the underlying geometry while consistently outperforming the non-analogical baselines.

As expected, prediction becomes increasingly difficult as the transfer task is made more challenging. Moving from age-based cohorts to finer demographic partitions and more specific movie categories results in a gradual increase of the Jensen-Shannon divergence for all methods. Nevertheless, the analogical models preserve their advantage over the baselines in the first five experimental settings. The final experiment, combining occupation, genre and decade, represents the most challenging configuration. In this case, the performance of the analogical models becomes comparable to that of the genre-only baseline, suggesting that when both cohorts and categories become highly specific, the amount of transferable information naturally decreases.

Overall, these experiments demonstrate that proportional analogies provide an effective mechanism for transferring preference distributions across demographic groups. The improvements are observed consistently across a broad range of experimental settings, indicating that the proposed analogical constructions capture meaningful structural relationships between categorical probability distributions.

%% file: sections/conclusion.tex
We introduced an intrinsic proportional analogy relation on Riemannian manifolds based on the notion of geodesic midpoint. The proposed construction satisfies the fundamental properties of proportional analogies while naturally reflecting the geometry of the underlying manifold. We established several theoretical characterizations of this relation, including equivalent formulations on geodesically convex manifolds and invariance under isometries. We further showed that, on Riemannian symmetric spaces, the analogy admits a particularly simple form, leading to closed-form characterizations on several important manifolds such as spheres, hyperbolic spaces, the manifold of symmetric positive definite matrices, shape spaces, and manifolds of probability distributions.

More generally, this work provides a unified geometric framework for extending proportional analogies beyond symbolic domains and Euclidean vector spaces. Several directions remain to be explored, including extensions to non-Riemannian geometries, the development of efficient numerical solvers on general manifolds, and empirical investigations of analogy-based learning methods. We believe that these developments could lead to new applications in transfer learning, meta-learning, sparse coding, and geometric data augmentation.